\documentclass{article}

\PassOptionsToPackage{numbers, compress}{natbib}

\usepackage[preprint]{neurips_2019}

\usepackage[utf8]{inputenc} 
\usepackage[T1]{fontenc}    
\usepackage{hyperref}       
\usepackage{url}            
\usepackage{booktabs}       
\usepackage{mathtools}
\usepackage{graphicx}
\usepackage{subcaption}
\usepackage{booktabs} 
\usepackage{amsmath}
\usepackage{amssymb}
\usepackage{amsfonts}       
\usepackage{nicefrac}       
\usepackage{microtype}      
\usepackage{algorithm}
\usepackage{algorithmicx,algpseudocode,setspace}
\usepackage{amsthm}
\usepackage{bbm}
\usepackage{xcolor}
\usepackage{caption}
\usepackage{verbatim}
\usepackage{wrapfig}
\newtheorem{theorem}{Theorem}
\newtheorem{assumption}{Assumption}

\newcommand{\param}{\beta}
\newcommand{\vr}{V^{\param}_{\theta, \omega}}
\newcommand{\vrb}{V^{\param}_{\theta,\omega}}
\newcommand{\s}{\mathcal{S}}
\newcommand{\p}{\mathcal{P}}
\newcommand{\pol}{\pi}
\newcommand{\A}{\mathcal{A}}
\newcommand{\R}{\mathbb{R}}
\newcommand{\T}{\mathcal{T}}
\newcommand{\V}{V_{\theta}}

\newcommand{\expect}{\mathop{\mathbb{E}}}
\newcommand{\RE}{r}

\newcommand{\norm}[1]{\left\lVert#1\right\rVert}

\title{Recurrent Value Functions}

\author{
Pierre Thodoroff {\thanks{Equal contribution} \space \thanks{McGill Univeristy - Mila, Montreal, Quebec}} \\
\texttt{pierthodo@gmail.com} \\
\And
Nishanth Anand {\footnotemark[1] \space \footnotemark[2]} \\
\texttt{nishanth.anand@mail.mcgill.ca} \\
\AND
Lucas Caccia {\footnotemark[2]}\\
\texttt{lucas.page-caccia@mail.mcgill.ca} \\
\And
Doina Precup {\footnotemark[2] \space \thanks{Deepmind, Montreal}}\\
\texttt{dprecup@cs.mcgill.ca} \\
\And
Joelle Pineau {\footnotemark[2] \space \thanks{Facebook AI Research, Montreal}}\\
\texttt{jpineau@cs.mcgill.ca} \\
}

\begin{document}

\maketitle

\begin{abstract}
Despite recent successes in Reinforcement Learning, value-based methods often suffer from high variance hindering performance. In this paper, we illustrate this in a continuous control setting where state of the art methods perform poorly whenever sensor noise is introduced. To overcome this issue, we introduce Recurrent Value Functions (RVFs) as an alternative to estimate the value function of a state. We propose to estimate the value function of the current state using the value function of past states visited along the trajectory. Due to the nature of their formulation, RVFs have a natural way of learning an emphasis function that selectively emphasizes important states. First, we establish RVF's asymptotic convergence properties in tabular settings. We then demonstrate their robustness on a partially observable domain and continuous control tasks. Finally, we provide a qualitative interpretation of the learned emphasis function.
\end{abstract}

\section{Introduction}
Model-free Reinforcement Learning (RL) is a widely used framework for sequential decision making in many domains such as robotics \cite{kober2013reinforcement,abbeel2010autonomous} and video games \cite{vinyals2017starcraft,mnih2013playing,mnih2016asynchronous}. However, its use in the real-world remains limited due, in part, to the high variance of value function estimates \cite{greensmith2004variance}, leading to poor sample complexity \cite{glascher2010states,kakade2003sample}. This phenomenon is exacerbated by the noisy conditions of the real-world \cite{fox2015taming,pendrith1994reinforcement}. Real-world applications remain challenging as they often involve noisy data such as sensor noise and partially observable environments.

The problem of disentangling signal from noise in sequential domains is not specific to Reinforcement Learning and has been extensively studied in the Supervised Learning literature. In this work, we leverage ideas from time series literature and Recurrent Neural Networks to address the robustness of value functions in Reinforcement Learning. We propose Recurrent Value Functions (RVFs): an exponential smoothing of the value function. The value function of the current state is defined as an exponential smoothing of the values of states visited along the trajectory where the value function of past states are summarized by the previous RVF.

However, exponential smoothing along the trajectory can result in a bias when the value function changes dramatically through the trajectory (non-stationarity). This bias could be a problem if the environment encounters sharp changes, such as falling of a cliff, and the estimates are heavily smoothed. To alleviate this issue, we propose to use exponential smoothing on value functions using a trainable state-dependent emphasis function which controls the smoothing coefficients. Intuitively, the emphasis function adapts the amount of emphasis required on the current value function and the past RVF to reduce bias with respect to the optimal value estimate. In other words, the emphasis function identifies important states in the environment. An important state can be defined as one where \emph{its value differs significantly from the previous values along the trajectory}. For example, when falling off a cliff, the value estimate changes dramatically, making states around the cliff more salient. This emphasis function serves a similar purpose to a gating mechanism in a Long Short Term Memory cell of a Recurrent Neural Network \cite{hochreiter1997long}. 

To summarize the contributions of this work, we introduce RVFs to estimate the value function of a state by exponentially smoothing the value estimates along the trajectory. RVF formulation leads to a natural way of learning an emphasis function which mitigates the bias induced by smoothing. We provide an asymptotic convergence proof in tabular settings by leveraging the literature on asynchronous stochastic approximation \cite{tsitsiklis1994asynchronous}. Finally, we perform a set of experiments to demonstrate the robustness of RVFs with respect to noise in continuous control tasks and provide a qualitative analysis of the learned emphasis function which provides interpretable insights into the structure of the solution.

\section{Technical Background}
A Markov Decision Process (MDP), as defined in \cite{puterman2014markov}, consists of a discrete set of states $\s$, a transition function $\p : \s \times \A \times \s \mapsto [0,1]$, and a reward function $\RE : \s \times \A \mapsto \R$. In each round $t$, the learner observes current state $s_t\in\s$ and selects an action $a_t\in\A$. As a response, it receives a reward $r_t = \RE(s_t,a_t)$ and moves to a new state $s_{t+1}\sim \p(\cdot|s_t, a_t)$. We define a stationary policy $\pol$ as a probability distribution over actions conditioned on states $\pol : \s \times \A \mapsto [0,1]$, such that $a_t \sim \pol(\cdot|s_t)$. 
In policy evaluation, the goal is to find the optimal value function $V^{\pol}$ that estimates the discounted expected return of a policy $\pol$ at a state $s\in\s$, $V^{\pol}(s) =\expect_{\pol}[\sum_{t=0}^{\infty} \gamma^t \RE_{t+1} | s_0 = s]$, with discount factor $\gamma \in [0,1)$. In this paper, we only consider policy evaluation and simplify the model: $r(s_t)=r(s_t,a_t)$.

In practice, $V^{\pol}$ is approximated using Monte Carlo rollouts \cite{suttonreinforcement} or TD methods \cite{sutton1988learning}. For example, the target used in TD(0) is $\expect_{s' \sim \pol}[r(s) + \gamma V^{\pol}(s')]]$. In Reinforcement Learning, the aim is to find a function $V_\theta: \mathbb{S} \rightarrow \mathbb{R}$ parametrized by $\theta$ that approximates $V^{\pol}$. We thus learn a set of parameters $\theta$ that minimizes the squared loss:
\begin{equation}\label{simple_loss}
    \mathcal{L}(\theta) = \expect_{\pol} [(V^{\pol} - V_\theta)^2],
\end{equation}
which yields the following update on the parameters $\theta$ by taking the derivative with respect to $\theta$:
\begin{equation}
    \theta_{t+1} = \theta_t + \alpha (V^{\pol}(s_t) - V_{\theta_t}(s_t)) \nabla_{\theta_t} \V(s_t),
\end{equation}
where $\alpha$ is a learning rate.

\section{Recurrent Value Functions (RVFs)}
\label{RLRL}
As mentioned earlier, performance of value-based methods are often heavily impacted by the quality of the data obtained \cite{fox2015taming,pendrith1994reinforcement}. For example, in robotics, noisy sensors are common and can significantly hinder performance of popular methods \cite{romoff2018reward}. In this work, we propose a method to improve the robustness of value functions by estimating the value of a state $s_t$ using the estimate at time step $t$ and the estimates of previously visited states $s_i$ where $i < t$.
Mathematically, the Recurrent Value Function (RVF) of a state $s$ at time step $t$ is given by:
\begin{equation}
\label{value}
\begin{split}
    V^{\param}(s_t) &= \param(s_t) V(s_t) + (1 - \param(s_t)) V^{\param}(s_{t-1}), \\ 
\end{split}
\end{equation}
where  $\param(s_t) \in [0,1]$. $V^{\param}$ estimates the value of a state $s_t$ as a convex combination of current estimate $V(s_t)$ and previous estimate $V^{\param}(s_{t-1})$. $V^{\param}(s_{t-1})$ can be recursively expanded further, hence the name Recurrent Value Function. $\param$ is the emphasis function which updates the recurrent value estimate. 

In contrast to traditional methods that attempt to minimize Eq. \ref{simple_loss}, the goal here is to find a set of parameters $\theta,\omega$ that minimize the following error:
\begin{equation}
\begin{split}
    \mathcal{L}(\theta,\omega) &= \expect_{\pol} [(V^{\pol} - \vr)^2], \\
	 \vr(s_t) &= \param_{\omega}(s_t) \V(s_t) + (1 - \param_{\omega}(s_t)) (\vr(s_{t-1})),
\end{split}
\label{loss_beta}
\end{equation}
where $\V$ is a function parametrized by $\theta$, and $\param_{\omega}$ is a function parametrized by $\omega$.
This error is similar to the traditional error in Eq. \ref{simple_loss}, but we replace the value function with $\vr$. In practice, $V^{\pol}$ can be any target such as TD(0), TD(N), TD($\lambda$) or Monte Carlo \cite{sutton1998reinforcement} which is used in Reinforcement Learning.
We minimize Eq. \ref{loss_beta} by updating $\theta$ and $\omega$ using the semi-gradient technique which results in the following update rule:
\begin{equation}
\begin{split}
	\theta &= \theta + \alpha \delta_t  \nabla_{\theta} \vr (s_t), \\
	\omega &= \omega + \alpha \delta_t  \nabla_{\omega} \vr (s_t),
\end{split}
\end{equation}
where $\delta_t = V^{\pol}(s_t) - \vr(s_t)$ is the TD error with RVF in the place of the usual value function. The complete algorithm using the above update rules can be found in Algorithm \ref{algo_rtd}.

\begin{algorithm}[H]
\caption{Recurrent Temporal Difference(0)}
\begin{spacing}{1.2}
\begin{algorithmic}[1]
    \label{RTD}
    \State Input: $\pi$,$\gamma$,$\theta$,$\omega$
    \State Initialize: $\vr(s_0) = \V(s_0)$
    \State Output: $\theta,\omega$ \Comment{Return the learned parameters}
    \For{t}
        \State Take action $a \sim \pi(s_t)$ , observe $r(s_t),s_{t+1}$
        \State $\vr(s_t) = \param_{\omega}(s_t) \V(s_t) + (1-\param_{\omega}(s_t)) \vr (s_{t-1})$ \Comment{Compute the RVF}
        \State $\delta_t = V^{\pol}(s_t) - \vr(s_t)$ \Comment{Compute TD error with respect to RVF}
        \State $\theta = \theta + \alpha \delta_t \nabla_{\theta} \vr(s_t) $ \Comment{Update parameters of the value function $\theta$}
        \State $\omega = \omega + \alpha \delta_t \nabla_{\omega} \vr(s_t)$ \Comment{Update parameters of the emphasis function $\omega$}
    \EndFor
\end{algorithmic}
\label{algo_rtd}
\end{spacing}
\end{algorithm}

As discussed earlier, $\param_{\omega}$ learns to identify states whose value significantly differs from previous estimates. While optimizing for the loss function described in Eq. \ref{loss_beta}, the $\param_{\omega}(s_t)$ learns to bring the RVF $\vr$ closer to the target $V^{\pol}$. It does so by placing greater emphasis on whichever is closer to the target, either $\V(s_t)$ or $\vr(s_{t-1})$. Concisely, the updated behaviour can be split into four scenarios. A detailed description of these behaviours is provided in Table \ref{sample-table}. Intuitively, if the past is not aligned with the future, $\param$ will emphasize the present. Likewise, if the past is aligned with the future, then $\param$ will place less emphasis on the present. This behaviour is further explored in the experimental section.

\begin{table}[h]
\caption{Behaviour of $\param$ based on the loss}
\centering
\begin{tabular}{c|c c } 
  & \small{$V^{\pol}(s_t)>\vr(s_t)$} & \small{$V^{\pol}(s_t)<\vr(s_t)$} \\ 
   \hline
 \small{$\V(s_t)>\vr(s_{t-1})$} & $\param \uparrow$ &  $\param \downarrow$ \\ 
 \small{$\V(s_t)<\vr(s_{t-1})$} & $\param \downarrow$ & $\param \uparrow$  \\ 
\end{tabular}
\vspace{7pt}
\label{sample-table}
\end{table}

Note that, the gradients of $\vr$ take a recursive form (gradient through time) as shown in Eq. \ref{grad_rvf}. The gradient form is similar to LSTM \cite{hochreiter1997long}, and GRU \cite{chung2014empirical} where $\param$ acts as a gating mechanism that controls the flow of gradient. LSTM uses a gated exponential smoothing function on the hidden representation to assign credit more effectively. In contrast, we propose to exponentially smooth the outputs (value functions) directly rather than the hidden state. This gradient can be estimated using backpropagation through time by recursively applying the chain rule where:
 \begin{equation}
\label{grad_rvf}
\nabla_{\theta} \vr (s_t) = \param_{\omega}(s_t) \nabla_{\theta} \V(s_t) + (1-\param_{\omega}(s_t)). \nabla_{\theta} \vr(s_{t-1})
\end{equation}
However, this can become computationally expensive in environments with a large episodic length, such as continual learning. Therefore, we could approximate the gradient $\nabla_{\theta}\vr(s_{t})$ using a recursive \emph{eligibility} trace:.
\begin{equation}
\label{trace}
    e_t = \param_{\omega}(s_t) \nabla_{\theta} \V (s_t) + (1-\param_{\omega}(s_t)) e_{t-1}.
\end{equation}

In the following section, we present the asymptotic convergence proof of RVF.

\subsection*{Asymptotic convergence}
For this analysis, we consider the simplest case: a tabular setting with TD(0) and a fixed set of $\param$. In the tabular setting, each component of $\theta$ and $\omega$ estimates one particular state, allowing us to simplify the notation. In this section, we simplify the notation by dropping $\theta$ and $\omega$ such that $\V(s_t) = V(s_t)$ and $\param_{\omega}(s_t) = \param_t$.
In the tabular setting, convergence to the fixed point of an operator is usually proven by casting the learning algorithm as a stochastic approximation \cite{tsitsiklis1994asynchronous,borkar2009stochastic,borkar2000ode} of the form:
\begin{equation}\label{stoch}
    \theta_{t+1} = \theta_t + \alpha( \T \theta_t  - \theta_t + w(t)), 
\end{equation}
where $\T: \mathbb{R^{|\mathbb{S}|}} \rightarrow \mathbb{R^{|\mathbb{S}|}}$ is a contraction operator and $w(t)$ is a noise term. 
The main idea is to cast the Recurrent Value Function as an asynchronous stochastic approximation \cite{tsitsiklis1994asynchronous} with an additional regularization term. By bounding the magnitude of this term, we show that the operator is a contraction. The algorithm is asynchronous because the eligibility trace only updates certain states at each time step. 

We consider the stochastic approximation formulation described in Eq. \ref{stoch} with the following operator $\T^{\param}: \mathbb{R^{|\mathbb{S}|}} \rightarrow \mathbb{R^{|\mathbb{S}|}}$ for any $i \leq t$:
\begin{equation}
    \T^{\param} V(s_i) = \expect_{\pol}[r_t + \gamma V(s_{t+1}) + \Delta_t(s_i)]
\end{equation}
for all states $s_i$ with $\param_i \in (0,1]$. $\Delta_t(s_i)$ can be interpreted as a regularization term composed of the difference between $V(s_i)$ and $V^{\param}(s_t)$. \\
To obtain this operator we first examine the update to $V(s_i)$ made during the trajectory at time step $t$:
\begin{equation}
\label{decompose}
\begin{split}
    V(s_i) &= V(s_i) + \alpha e_t(s_i) ( r_t + \gamma V(s_{t+1}) - V^{\param}(s_t)) \\
    &= V(s_i) + \alpha e_t(s_i) ( r_t + \gamma V(s_{t+1}) + \Delta_t(s_i) - V(s_i))
\end{split}
\end{equation}
where $\Delta_t(s_i) = (1-C_t(s_i))(V(s_i) - \widetilde{V}_t(s_i) )$ and  $C_t(s_i) = \param_i \prod_{p=i+1}^t (1-\param_p)$.
$\widetilde{V}_t(s_t)$ is a convex combination of all $V$ encountered in the trajectory, with the exception of $V(s_i)$, weighted by their respective contribution($\param$) to the estimate $V^{\param}(s_t)$. For example, if we consider updating $V(s_2)$ at $t=3$ and have the following $\param_1 = 0.9,\param_2 = 0.1,\param_3 = 0.1$, the value of $\widetilde{V}_3(s_2)$ will be mainly composed of $V(s_1)$. The main component of the error will be $r_t + \gamma V(s_4) - V(s_1)$. An example on how to obtain this decomposition can be found in the section \ref{decompose_app} of Appendix. 
In practice, one can observe an increase in the magnitude of this term with a decrease in \emph{eligibility}. This suggests that the biased updates contribute less to the learning. Bounding the magnitude of $\Delta$ to ensure contraction is the key concept used in this paper to ensure asymptotic convergence.

We consider the following assumptions to prove convergence:
The first assumption deals with the ergodic nature of the Markov chain. It is a common assumption in theoretical Reinforcement Learning that guarantees an infinite number of visits to all states, thereby avoiding chains with transient states.
\begin{assumption}
The Markov chain is ergodic.
\end{assumption}

The second assumption concerns the relative magnitude of the maximum and minimum reward, and allow us to bound the magnitude of the regularization term. 
\begin{assumption}
\label{assumption_rew}
We define $R_{\max}$ and $R_{\min}$ as the maximum and minimum reward in an MDP. All rewards are assumed to be positive and scaled in the range $[R_{\min},\widetilde{R}_{\max}]$ such that the scaled maximum reward $\widetilde{R}_{\max}$ satisfies the following:
\begin{equation}
    D\widetilde{R}_{\text{max}}  \leq R_{\text{min}} \qquad  D > \gamma 
\end{equation}
where $D \in (0.5,1]$ is a constant to be defined based on $\gamma$.
\end{assumption}
In theory, scaling the reward is reasonable as it does not change the optimal solution of the MDP \cite{van2016learning}. In practice, however, this may be constraining as the range of the reward may not be known beforehand. This assumption could be relaxed by considering the trajectory's information to bound $\Delta$. As an example, one could consider any physical system where transitions in the state space are smooth (continuous state space) and bounded by some Lipschitz constant in a similar manner than \cite{shah2018q}. 

As mentioned earlier, the key component of the proof is to control the magnitude of the term in Eq. \ref{decompose}: 
$\Delta_t(s_i) = (1-C_t(s_i))(V(s_i) - \widetilde{V}_t(s_i))$. As the eligibility of this update gets smaller, the magnitude of the term gets bigger. This suggests that not updating certain states whose eligibility is less than the threshold \emph{C} can help mitigate biased updates.
Depending on the values of $\gamma$ and $D$, we may need to set a threshold $C$ to guarantee convergence.
\begin{theorem}
\label{contraction_theorem}
Define $V_{\max} = \frac{\widetilde{R}_{\max}}{1-(\gamma+(1-D))}$ and $V_{\min} = \frac{R_{\min}}{1-(\gamma-(1-D))}$. $\T^{\param}: X \rightarrow X$ is a contraction operator if the following holds:
\begin{itemize}
    \item Let $X$ be the set of $V$ functions such that $\forall s \in \mathbb{S} \quad  V_{\min} \leq V(s) \leq V_{\max}$.  The functions $V$ are initialized in $X$.
    \item For a given $D$ and $\gamma$ we select $C$ such that $\Delta \leq (1-C)(V_{\max}-V_{\min}) \leq (1-D)V_{\min}$.
\end{itemize}
\end{theorem}
We outline important details of the proof here. A full version can be found in Appendix \ref{proof}. For two sets of value functions $U,V \in X$:
\begin{equation}
    \max_{s}\expect_{\pol}[ \Delta^V(s) -  \Delta^U(s))] \leq \max_{s}\expect_{\pol}[(1-D)(V(s)-U(s))],\\
\end{equation}
where $\Delta^U$ is the $\Delta$ described in Eq. \ref{decompose} with value function $U$. For any $\gamma$ contraction operator, we can now guarantee that $\T^{\param}$ is a $\gamma + (1-D)$ contraction operator, where $\gamma + (1-D) < 1$ holds from Assumption \ref{assumption_rew}. We provide an example in Appendix \ref{set_gamma} to set $C$ based on $\gamma$ and $D$.
We can guarantee that $\V$ converges to a fixed point of the operator $\T^{\param}$ with probability $=1$ using Theorem 3 of \cite{tsitsiklis1994asynchronous}. The assumptions of Theorem 3 of \cite{tsitsiklis1994asynchronous} are discussed in section \ref{assumption} of Appendix.


\section{Related work}

One important similarity of RVFs is with respect to the online implementation of $\lambda$ return \cite{sutton1998reinforcement,dayan1992convergence}. Both RVF and online $\lambda$ returns have an eligibility trace form, but the difference is in RVF's capacity to ignore a state based on $\param$. In this paper we argue that this can provide more robustness to noise and partial observability.  The ability of RVF to emphasize a state is similar to the interest function in emphatic TD \cite{mahmood2015emphatic}, however, learning a state-dependant interest function and $\lambda$ remains an open problem. In contrast, RVF has a natural way of learning $\param$ by comparing the past and the future. The capacity to ignore states shares some motivations to semi-Markov decision process \cite{puterman1990markov}. Learning $\param$ and ignoring states can be interpreted as learning temporal abstraction over the trajectory in policy evaluation.
In reward shaping literature, several works such as Temporal Value Transport \cite{TVT}, Temporal Regularization \cite{thodoroff2018temporal}, Natural Value Approximator \cite{xu2017natural} attempt to modify the target to either enforce temporal consistency or to assign credit efficiently. This departs from our formulation as we consider estimating a value function directly by using the previous estimates rather than by  modifying the target. As a result of modifying the estimate, RVFs can choose to ignore a gradient while updating, which is not possible in other works. For example, in settings where the capacity is limited, updating on noisy states can be detrimental for learning.
Finally, RVF can also be considered as a partially observable method \cite{kaelbling1998planning}. However, it differs significantly from the literature as it does not attempt to infer the underlying hidden state explicitly, but rather only decides if the past estimates align with the target. We argue that inferring an underlying state may be significantly harder than learning to ignore or emphasize a state based on its value. This is illustrated in the next section.

\section{Experiments}
In this section, we perform experiments on various tasks to demonstrate the effectiveness of RVF.
First, we explore RVF robustness to partial observability on a synthetic domain. We then showcase RVF's robustness to noise on several complex continuous control tasks from the Mujoco suite \cite{todorov2012mujoco}. An example of policy evaluation is also provided in Appendix \ref{Policy_evaluation_RVF} as a reading.
\subsection{Partially observable multi-chain domain}
We consider the simple chain MDP described in Figure \ref{fig:toy MDP}. This MDP has three chains connected together to form a \emph{Y}. Each of the three chains (left of $S_1$, right of $S_2$, right of $S_3$) is made up of a sequence of states. The agent starts at $S_0$ and navigates through the chain. At the intersection $S_1$, there is a $0.5$ probability to go up or down. The chain on top receives a reward of $+1$ while the one at the bottom receives a reward of $-1$. Every other transition has a reward of $0$, unless specified otherwise.
\begin{figure}[h]
\centering
\begin{subfigure}[b]{.45\textwidth}    \includegraphics[width=0.9\textwidth,height=3.5cm]{fig/toy_MDP.png}
    \caption{Simple chain MDP.}
    \label{fig:toy MDP}
\end{subfigure}
\begin{subfigure}[b]{.45\textwidth}
    \centering
    \includegraphics[width=0.95\textwidth,height=3.75cm]{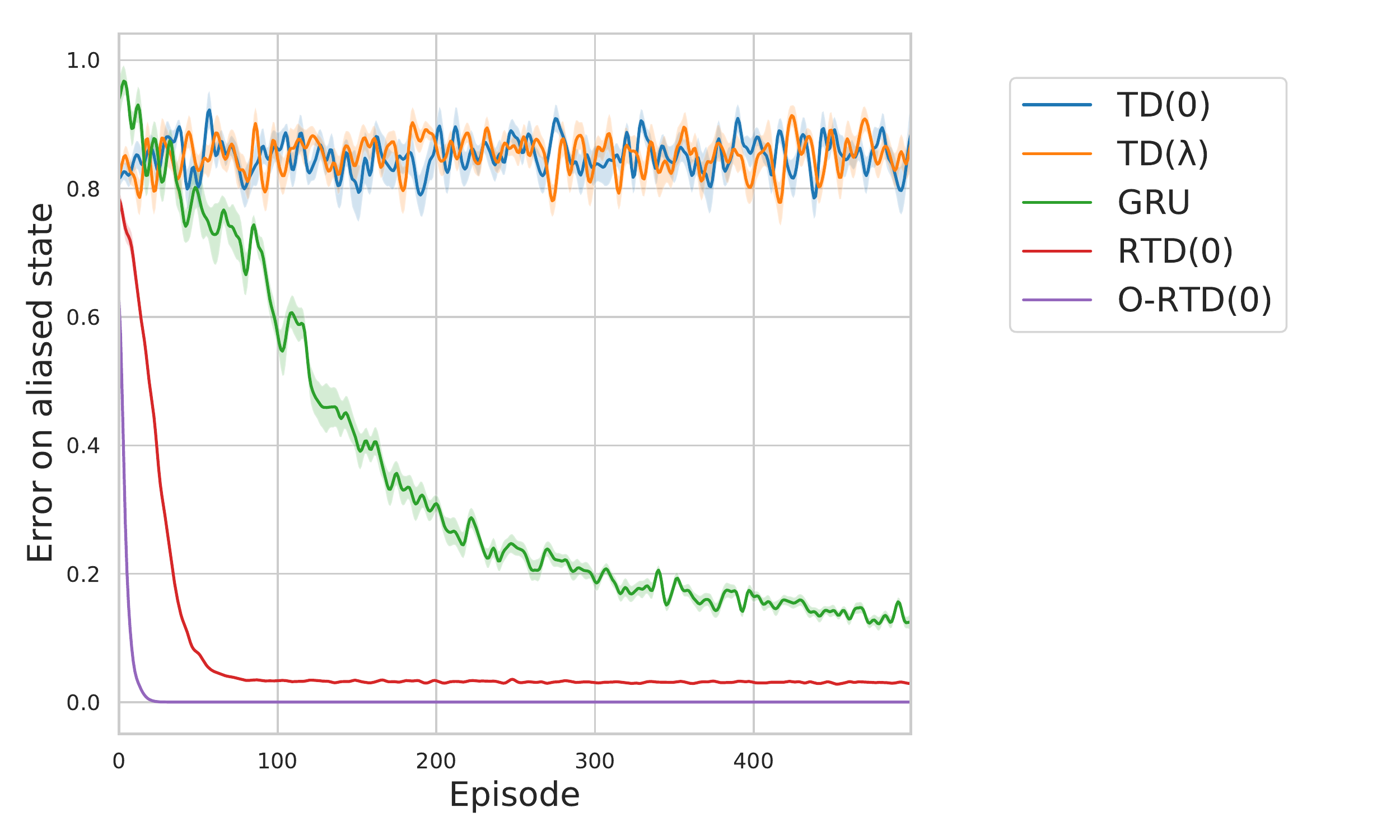}
    \caption{Results on the aliased Y-chain.}
    \label{fig:pomdp}
\end{subfigure}
\caption{(a) Simple chain MDP: The agent starts at state $S_0$ and navigates along the chain. States $S_4$ and $S_5$ are aliased. (b) Results on the aliased Y-chain on various methods, such as TD$(0)$, TD($\lambda$), GRU, RTD($0$), and Optimal RTD($0$) (O-RTD(0)) averaged over 20 random seeds.}
\end{figure}
We explore the capacity of recurrent learning to solve a partially observable task in the \emph{Y} chain. In particular, we consider the case where some states are aliased (share a common representation). The representation of the states $S_4$ and $S_5$ in Figure \ref{fig:toy MDP} are aliased. The goal of this environment is to correctly estimate the value of the aliased state $V^{\pol}(S_4)=0.9,V^{\pol}(S_5)=-0.9$ (due to the discount factor(0.9) and the length of each chain being 3). When TD methods such as TD($0$) or TD($\lambda$) are used, the values of the aliased states $S_4$ and $S_5$ are close to $0$ as the reward at the end of the chain is $+1$ and $-1$.  However, when learning $\param$ (emphasis function $\param$ is modelled using a sigmoid function), Recurrent Value Functions achieve almost no error in their estimate of the aliased states as illustrated in Figure \ref{fig:pomdp}. This can be explained by observing that $\param \rightarrow 0$ on the aliased state due to the fact that the previous values along the trajectory are better estimates of the future than those of the aliased state. As $\param \to 0$, $\vr(S_4)$ and $\vr(S_5)$ tend to rely on their more accurate previous estimates, $\vr(S_2)$ and $\vr(S_3)$. We see that learning to ignore certain states can at times be sufficient to solve an aliased task. We also compare with a recurrent version (O-RTD) where optimal values of $\param$ are used. In this setting,  $\param(S_1)=\param(S_2) = \param(s_3) = 1 $ and other states have $\param = 0$. Another interesting observation is with respect to Recurrent Neural Networks. RNNs are known to solve tasks which have partial observability by inferring the underlying state. LSTM and GRU have many parameters that are used to infer the hidden state. Correctly learning to keep the hidden state intact can be sample-inefficient. In comparison, $\param$ can estimate whether or not to put emphasis (confidence) on a state value using a single parameter. This is illustrated in Figure \ref{fig:pomdp} where RNNs take 10 times more episodes to learn the optimal value when compared to RVF. This illustrates a case where learning to ignore a state is easier than inferring its hidden representation. The results displayed in Figure \ref{fig:pomdp} are averaged over 20 random seeds. For every method, a hyperparameter search is done to obtain their optimal value. These can be found in Appendix \ref{hyper_toy}. We noticed that the emphasis function is easier to learn if the horizon of the target is longer, since a longer horizon provides a better prediction of the future. To account for this, we use $\lambda$-return as a target.
\subsection{Deep Reinforcement Learning}
Next, we test RVF on several environments of the Mujoco suite \cite{todorov2012mujoco}. We also evaluate the robustness of different algorithms by adding $\epsilon$ sensor noise (drawn from a normal distribution $\epsilon \sim N(0,1)$) to the observations as presented in \cite{zhang2018dissection}.
We modify the critic of A2C \cite{wu2017scalable} (R-A2C) and Proximal Policy Optimization (R-PPO) \cite{schulman2017proximal} to estimate the recurrent value function parametrized by $\theta$. We parametrize $\param$ using a seperate network with the same architecture as the value function (parametrized by $\omega$). We minimize the loss mentioned in Eq. \ref{loss_beta} but replace the target with generalized advantage function ($\V^{\lambda}$) \cite{schulman2015high} for PPO and TD($n$) for A2C. Using an automatic differentiation library (Pytorch \cite{paszke2017automatic}), we differentiate the loss through the modified estimate to learn $\theta$ and $\omega$. The default optimal hyperparameters of PPO and A2C are used. Due to the batch nature of PPO, obtaining the trajectory information to create the computational graph can be costly. In this regard, we cut the backpropagation after $N$ timesteps in a similar manner to truncated backpropagation through time. The number of backpropagation steps is obtained using hyperparameter search. Details can be found in Appendix \ref{deep_RL_appendix}. We use a truncated backprop of $N=5$ in our experiments as we found no empirical improvements for $N=10$. For a fairer comparison in the noisy case, we also compare the performance of two versions of PPO with an LSTM. The first version processes one trajectory every update. The second uses a buffer in a similar manner to PPO, but the gradient is cut after 5 steps as the computation overhead from building the graph every time is too large. The performance reported is averaged over 20 different random seeds with a confidence interval of $68\%$ displayed \footnote{The base code used to develop this algorithm can be found here \cite{pytorchrl}.}
\subsubsection{Performance}
\begin{figure}[H]
    \centering
    \includegraphics[width=\textwidth,height=6.9cm]{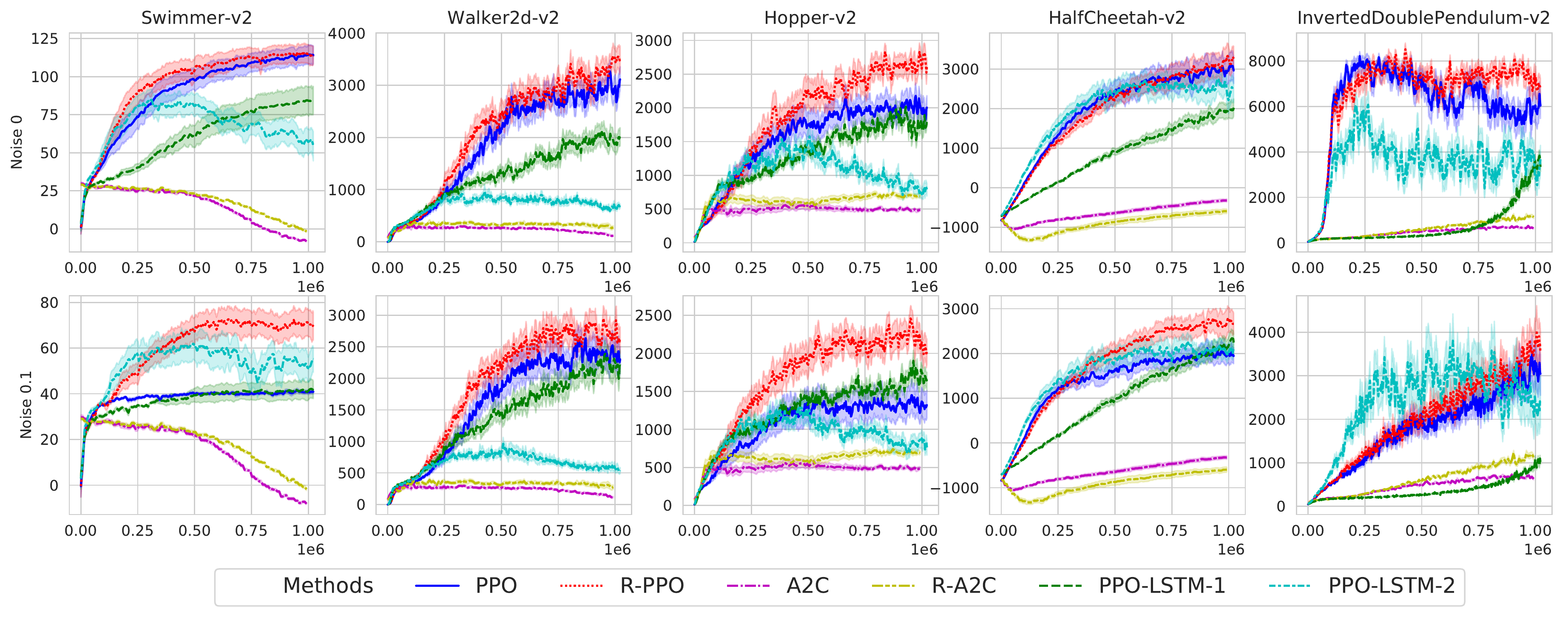}
    \caption{Performance on Mujoco tasks. Results on the first row are generated without noise  and on the second row by inducing a Gaussian noise ($\epsilon \sim \mathbb{N}(0,0.1)$) in the sensor inputs.}
    \label{fig:perf_mujoco_final}
\end{figure}
As demonstrated in Figure \ref{fig:perf_mujoco_final}, we observe a marginal increase in performance on several tasks such as Swimmer, Walker, Hopper, Half Cheetah and Double Inverted Pendulum in the fully observable setting. However, severe drops in performance were observed in the vanilla PPO when we induced partial observability by introducing a Gaussian noise to the observations. On the other hand, R-PPO (PPO with RVF) was found to be robust to the noise, achieving significantly higher performance in all the tasks considered. In both cases, R-PPO outperforms the partially observable models (LSTM). The mean and standard deviation of the emphasis function for both noiseless and noisy versions can be found in Appendix(\ref{fig:beta_mean}, \ref{fig:beta_std}). At the same time, A2C performance on both vanilla and recurrent versions (referred as R-A2C) were found to be poor. We increased the training steps on both versions and noticed the same observations as mentioned above once A2C started to learn the task. The performance plots, along with the mean and standard deviation of the emphasis function during training, can be found in Appendix (\ref{fig:a2c_no_noise_full}, \ref{fig:a2c_noise_full}, \ref{fig:beta_mean_a2c_mujoco},  \ref{fig:beta_std_a2c_mujoco}).
\subsubsection{Qualitative interpretation of the emphasis function $\param$ }
\emph{Hopper}: At the end of training, we can qualitatively analyze the emphasis function ($\param$) through the trajectory. We observe cyclical behaviour shown in Figure \ref{fig:cyle_hopper}, where different colours describe various stages of the cycle. The emphasis function learned to identify \emph{important states} and to ignore the others. One intuitive way to look at the emphasis function($\param$) is: \emph{If I were to give a different value to a state, would that alter my policy significantly?} We observe an increase in the value of the emphasis function ($\param$) when the agent must make an important decision, such as jumping or landing. We see a decrease in the value of the emphasis function ($\param$) when the agent must perform a trivial action. This pattern is illustrated in Figure \ref{fig:visual_hopper} and \ref{fig:cyle_hopper}. This behaviour is cyclic and repetitive, a video of which can be found in the following link\footnote{\url{https://youtu.be/0bzEcrxNwRw}{}}.
\begin{figure}[h]
\centering
\begin{subfigure}{.45\textwidth}
    \centering
    \begin{minipage}[b]{0.4\linewidth}
    \centering
    \includegraphics[width=\textwidth,height=1.9cm]{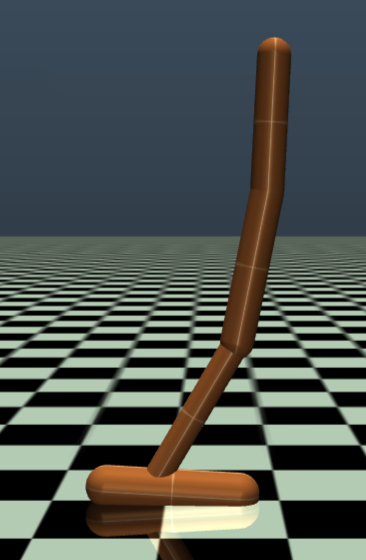}
    Phase 1: high $\param$
    \label{fig:phase_0}
    \end{minipage}
    \hspace{0.05cm}
    \begin{minipage}[b]{0.4\linewidth}
    \centering
    \includegraphics[width=\textwidth,height=1.9cm]{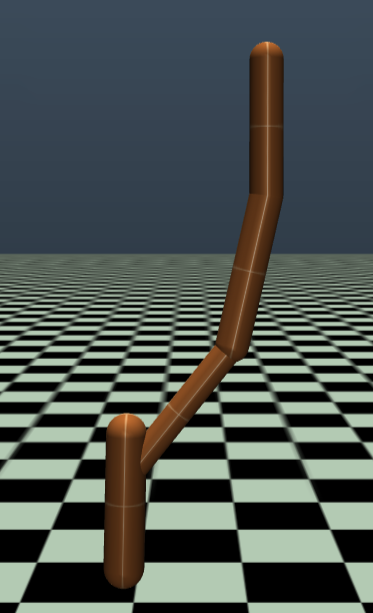}
    Phase 2: low $\param$
    \label{fig:phase_1}
    \end{minipage}
    \hspace{0.02cm}
    \begin{minipage}[b]{0.4\linewidth}
    \centering
    \includegraphics[width=\textwidth,height=1.9cm]{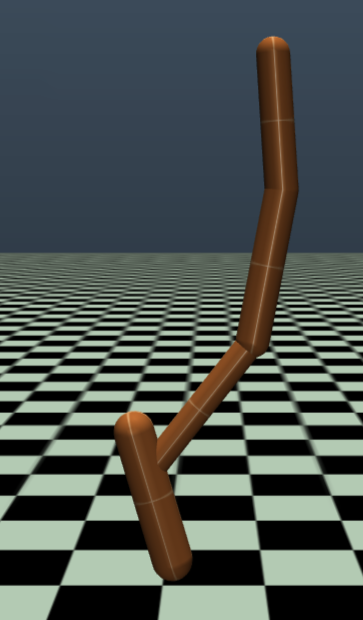}
    Phase 3: high $\param$
    \label{fig:phase_2}
    \end{minipage}
    \vspace{0.02cm}
    \hspace{0.02cm}
    \begin{minipage}[b]{0.4\linewidth}
    \centering
    \includegraphics[width=\textwidth,height=1.9cm]{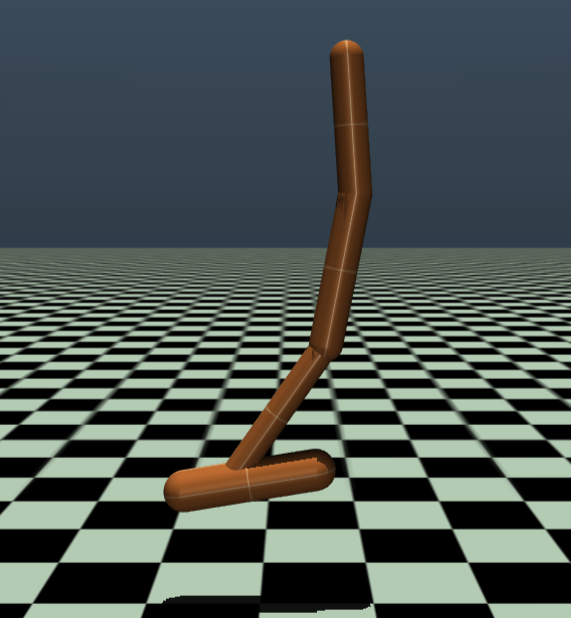}
     Phase 4: low $\param$
    \label{fig:phase_3}
    \end{minipage}
    \hspace{0.02cm}
    \caption{Cyclical behaviour of $\param$ on Hopper.}
    \label{fig:visual_hopper}
\end{subfigure}
\begin{subfigure}{.45\textwidth}
    \centering
    \includegraphics[width=\textwidth,height=4.7cm]{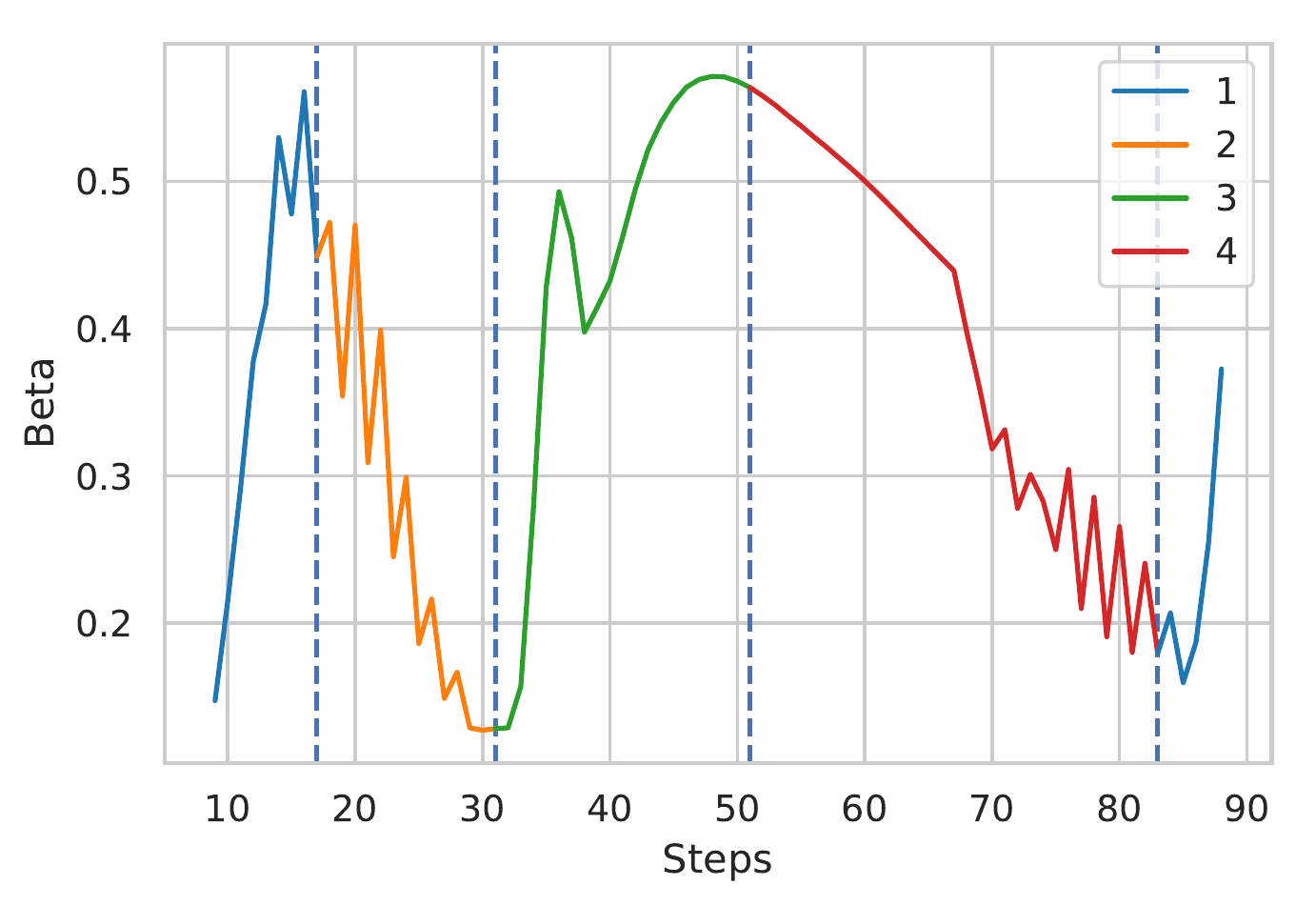}
    \caption{Behaviour of $\param$ through the trajectory.}
    \label{fig:cyle_hopper}
\end{subfigure}
\caption{(a) The emphasis function learns to emphasize key states in the environment. The emphasis function is high when the agent is making important decisions, such as landing or taking off (Phase 1 and Phase 3). The emphasis function is low when the agent is making decisions while it is in the air (Phase 2 and Phase 4). (b) Behaviour of the emphasis function along the trajectory for various phases described in (a) for one period. The emphasis function keeps repeating the behaviour.}
\end{figure}
\section{Discussions and Future Work}
\paragraph{Temporal Credit assignment:} As mentioned earlier, we can control the flow of gradient by using emphasis function $\param_{\omega}(s_t)$ and pass gradient to the states that contributed to the reward but are located several time-steps earlier. We could potentially do credit assignment on states that are temporally far away by forcing the emphasis function between these states to be close to $0$. This setting could be useful in problems with long horizons, such as lifelong learning and continual learning.
\paragraph{$\param$ as an interest function:}  In Reinforcement Learning, having access to a function quantifying the \emph{interest} \cite{mahmood2015emphatic} of a state can be helpful. For example, one could decide to explore from those states, prioritize experience replay based on those states, and use $\param$ to set the $\lambda$ to bootstrap from interesting states. Indeed, bootstrapping on states with a similar value (low $\param$) than the one estimated will only result in variance. The most informative updates come from bootstrapping on states with different values (high $\param$). We also believe $\param$ to be related to the concepts of bottleneck states \cite{tishby2011information} and reversibility.
\paragraph{Partially observable domain: } As demonstrated earlier, RVFs are able to correctly estimate the value of an aliased/noisy state using the trajectory's estimate. We believe that this is a promising area to explore because, as the experiments suggest, ignoring an uninformative state can sometimes suffice to learn its value function. This is in contrast to traditional POMDP methods which attempt to infer the belief state. Smoothing the output with a gating mechanism could also be useful for sequential tasks in Supervised Learning, such as regression or classification. 
\paragraph{Adjusting for the reward:} In practice, some environments in Reinforcement Learning have
a constant reward at every time step, potentially inducing bias in $\vr$ estimates. It would be possible to modify the RVF formulation to account for the reward that was just seen, such that $\vr(s_t) = \param V(s_t) + (1-\param) (\vr(s_{t-1}) - r_{t-1})$.
Whether or not subtracting the reward can reduce the bias will depend on the environment considered.
\paragraph{Conclusion:} In this work we propose Recurrent Value Functions to address variance issues in Model-free Reinforcement Learning. First, we prove the asymptotic convergence of the proposed method. We then demonstrate the robustness of RVF to noise and partial observability in a synthetic example and on several tasks from the Mujoco suite. Finally, we describe the behaviour of the emphasis function qualitatively.
\clearpage

\bibliography{library}

\begin{thebibliography}{38}
\providecommand{\natexlab}[1]{#1}
\providecommand{\url}[1]{\texttt{#1}}
\expandafter\ifx\csname urlstyle\endcsname\relax
  \providecommand{\doi}[1]{doi: #1}\else
  \providecommand{\doi}{doi: \begingroup \urlstyle{rm}\Url}\fi

\bibitem[Abbeel et~al.(2010)Abbeel, Coates, and Ng]{abbeel2010autonomous}
P.~Abbeel, A.~Coates, and A.~Y. Ng.
\newblock Autonomous helicopter aerobatics through apprenticeship learning.
\newblock \emph{The International Journal of Robotics Research}, 29\penalty0
  (13):\penalty0 1608--1639, 2010.

\bibitem[Borkar(2009)]{borkar2009stochastic}
V.~S. Borkar.
\newblock \emph{Stochastic approximation: a dynamical systems viewpoint},
  volume~48.
\newblock Springer, 2009.

\bibitem[Borkar and Meyn(2000)]{borkar2000ode}
V.~S. Borkar and S.~P. Meyn.
\newblock The ode method for convergence of stochastic approximation and
  reinforcement learning.
\newblock \emph{SIAM Journal on Control and Optimization}, 38\penalty0
  (2):\penalty0 447--469, 2000.

\bibitem[Chung et~al.(2014)Chung, Gulcehre, Cho, and
  Bengio]{chung2014empirical}
J.~Chung, C.~Gulcehre, K.~Cho, and Y.~Bengio.
\newblock Empirical evaluation of gated recurrent neural networks on sequence
  modeling.
\newblock \emph{arXiv preprint arXiv:1412.3555}, 2014.

\bibitem[Chung et~al.(2018)Chung, Nath, Joseph, and White]{chung2018two}
W.~Chung, S.~Nath, A.~Joseph, and M.~White.
\newblock Two-timescale networks for nonlinear value function approximation.
\newblock 2018.

\bibitem[Dayan(1992)]{dayan1992convergence}
P.~Dayan.
\newblock The convergence of td ($\lambda$) for general $\lambda$.
\newblock \emph{Machine learning}, 8\penalty0 (3-4):\penalty0 341--362, 1992.

\bibitem[Fox et~al.(2015)Fox, Pakman, and Tishby]{fox2015taming}
R.~Fox, A.~Pakman, and N.~Tishby.
\newblock Taming the noise in reinforcement learning via soft updates.
\newblock \emph{arXiv preprint arXiv:1512.08562}, 2015.

\bibitem[Gl{\"a}scher et~al.(2010)Gl{\"a}scher, Daw, Dayan, and
  O'Doherty]{glascher2010states}
J.~Gl{\"a}scher, N.~Daw, P.~Dayan, and J.~P. O'Doherty.
\newblock States versus rewards: dissociable neural prediction error signals
  underlying model-based and model-free reinforcement learning.
\newblock \emph{Neuron}, 66\penalty0 (4):\penalty0 585--595, 2010.

\bibitem[Greensmith et~al.(2004)Greensmith, Bartlett, and
  Baxter]{greensmith2004variance}
E.~Greensmith, P.~L. Bartlett, and J.~Baxter.
\newblock Variance reduction techniques for gradient estimates in reinforcement
  learning.
\newblock \emph{Journal of Machine Learning Research}, 5\penalty0
  (Nov):\penalty0 1471--1530, 2004.

\bibitem[Hochreiter and Schmidhuber(1997)]{hochreiter1997long}
S.~Hochreiter and J.~Schmidhuber.
\newblock Long short-term memory.
\newblock \emph{Neural computation}, 9\penalty0 (8):\penalty0 1735--1780, 1997.

\bibitem[Hung et~al.(2018)Hung, Lillicrap, Abramson, Wu, Mirza, Carnevale,
  Ahuja, and Wayne]{TVT}
C.-C. Hung, T.~Lillicrap, J.~Abramson, Y.~Wu, M.~Mirza, F.~Carnevale, A.~Ahuja,
  and G.~Wayne.
\newblock Optimizing agent behavior over long time scales by transporting
  value, 2018.

\bibitem[Kaelbling et~al.(1998)Kaelbling, Littman, and
  Cassandra]{kaelbling1998planning}
L.~P. Kaelbling, M.~L. Littman, and A.~R. Cassandra.
\newblock Planning and acting in partially observable stochastic domains.
\newblock \emph{Artificial intelligence}, 101\penalty0 (1-2):\penalty0 99--134,
  1998.

\bibitem[Kakade et~al.(2003)]{kakade2003sample}
S.~M. Kakade et~al.
\newblock \emph{On the sample complexity of reinforcement learning}.
\newblock PhD thesis, 2003.

\bibitem[Kober et~al.(2013)Kober, Bagnell, and Peters]{kober2013reinforcement}
J.~Kober, J.~A. Bagnell, and J.~Peters.
\newblock Reinforcement learning in robotics: A survey.
\newblock \emph{The International Journal of Robotics Research}, 32\penalty0
  (11):\penalty0 1238--1274, 2013.

\bibitem[Kostrikov(2018)]{pytorchrl}
I.~Kostrikov.
\newblock Pytorch implementations of reinforcement learning algorithms.
\newblock \url{https://github.com/ikostrikov/pytorch-a2c-ppo-acktr}, 2018.

\bibitem[Mahmood et~al.(2015)Mahmood, Yu, White, and
  Sutton]{mahmood2015emphatic}
A.~R. Mahmood, H.~Yu, M.~White, and R.~S. Sutton.
\newblock Emphatic temporal-difference learning.
\newblock \emph{arXiv preprint arXiv:1507.01569}, 2015.

\bibitem[Mnih et~al.(2013)Mnih, Kavukcuoglu, Silver, Graves, Antonoglou,
  Wierstra, and Riedmiller]{mnih2013playing}
V.~Mnih, K.~Kavukcuoglu, D.~Silver, A.~Graves, I.~Antonoglou, D.~Wierstra, and
  M.~Riedmiller.
\newblock Playing atari with deep reinforcement learning.
\newblock \emph{arXiv preprint arXiv:1312.5602}, 2013.

\bibitem[Mnih et~al.(2016)Mnih, Badia, Mirza, Graves, Lillicrap, Harley,
  Silver, and Kavukcuoglu]{mnih2016asynchronous}
V.~Mnih, A.~P. Badia, M.~Mirza, A.~Graves, T.~Lillicrap, T.~Harley, D.~Silver,
  and K.~Kavukcuoglu.
\newblock Asynchronous methods for deep reinforcement learning.
\newblock In \emph{International Conference on Machine Learning}, pages
  1928--1937, 2016.

\bibitem[Paszke et~al.(2017)Paszke, Gross, Chintala, Chanan, Yang, DeVito, Lin,
  Desmaison, Antiga, and Lerer]{paszke2017automatic}
A.~Paszke, S.~Gross, S.~Chintala, G.~Chanan, E.~Yang, Z.~DeVito, Z.~Lin,
  A.~Desmaison, L.~Antiga, and A.~Lerer.
\newblock Automatic differentiation in pytorch.
\newblock In \emph{NIPS-W}, 2017.

\bibitem[Pendrith()]{pendrith1994reinforcement}
M.~D. Pendrith.
\newblock \emph{On reinforcement learning of control actions in noisy and
  non-Markovian domains}.
\newblock Citeseer.

\bibitem[Puterman(1990)]{puterman1990markov}
M.~L. Puterman.
\newblock Markov decision processes.
\newblock \emph{Handbooks in operations research and management science},
  2:\penalty0 331--434, 1990.

\bibitem[Puterman(1994)]{puterman2014markov}
M.~L. Puterman.
\newblock \emph{Markov decision processes: discrete stochastic dynamic
  programming}.
\newblock John Wiley \& Sons, 1994.

\bibitem[Romoff et~al.(2018)Romoff, Henderson, Pich{\'e}, Francois-Lavet, and
  Pineau]{romoff2018reward}
J.~Romoff, P.~Henderson, A.~Pich{\'e}, V.~Francois-Lavet, and J.~Pineau.
\newblock Reward estimation for variance reduction in deep reinforcement
  learning.
\newblock \emph{arXiv preprint arXiv:1805.03359}, 2018.

\bibitem[Schulman et~al.(2015)Schulman, Moritz, Levine, Jordan, and
  Abbeel]{schulman2015high}
J.~Schulman, P.~Moritz, S.~Levine, M.~Jordan, and P.~Abbeel.
\newblock High-dimensional continuous control using generalized advantage
  estimation.
\newblock \emph{arXiv preprint arXiv:1506.02438}, 2015.

\bibitem[Schulman et~al.(2017)Schulman, Wolski, Dhariwal, Radford, and
  Klimov]{schulman2017proximal}
J.~Schulman, F.~Wolski, P.~Dhariwal, A.~Radford, and O.~Klimov.
\newblock Proximal policy optimization algorithms.
\newblock \emph{arXiv preprint arXiv:1707.06347}, 2017.

\bibitem[Shah and Xie(2018)]{shah2018q}
D.~Shah and Q.~Xie.
\newblock Q-learning with nearest neighbors.
\newblock \emph{arXiv preprint arXiv:1802.03900}, 2018.

\bibitem[Sutton(1988)]{sutton1988learning}
R.~S. Sutton.
\newblock Learning to predict by the methods of temporal differences.
\newblock \emph{Machine learning}, 3\penalty0 (1):\penalty0 9--44, 1988.

\bibitem[Sutton and Barto()]{suttonreinforcement}
R.~S. Sutton and A.~G. Barto.
\newblock Reinforcement learning: An introduction.

\bibitem[Sutton and Barto(1998)]{sutton1998reinforcement}
R.~S. Sutton and A.~G. Barto.
\newblock \emph{Reinforcement learning: An introduction}, volume~1.
\newblock MIT press Cambridge, 1998.

\bibitem[Thodoroff et~al.(2018)Thodoroff, Durand, Pineau, and
  Precup]{thodoroff2018temporal}
P.~Thodoroff, A.~Durand, J.~Pineau, and D.~Precup.
\newblock Temporal regularization for markov decision process.
\newblock In \emph{Advances in Neural Information Processing Systems}, pages
  1782--1792, 2018.

\bibitem[Tishby and Polani(2011)]{tishby2011information}
N.~Tishby and D.~Polani.
\newblock Information theory of decisions and actions.
\newblock pages 601--636, 2011.

\bibitem[Todorov et~al.(2012)Todorov, Erez, and Tassa]{todorov2012mujoco}
E.~Todorov, T.~Erez, and Y.~Tassa.
\newblock Mujoco: A physics engine for model-based control.
\newblock In \emph{Intelligent Robots and Systems (IROS), 2012 IEEE/RSJ
  International Conference on}, pages 5026--5033. IEEE, 2012.

\bibitem[Tsitsiklis(1994)]{tsitsiklis1994asynchronous}
J.~N. Tsitsiklis.
\newblock Asynchronous stochastic approximation and q-learning.
\newblock \emph{Machine learning}, 16\penalty0 (3):\penalty0 185--202, 1994.

\bibitem[van Hasselt et~al.(2016)van Hasselt, Guez, Hessel, Mnih, and
  Silver]{van2016learning}
H.~P. van Hasselt, A.~Guez, M.~Hessel, V.~Mnih, and D.~Silver.
\newblock Learning values across many orders of magnitude.
\newblock In \emph{Advances in Neural Information Processing Systems}, pages
  4287--4295, 2016.

\bibitem[Vinyals et~al.(2017)Vinyals, Ewalds, Bartunov, Georgiev, Vezhnevets,
  Yeo, Makhzani, K{\"u}ttler, Agapiou, Schrittwieser,
  et~al.]{vinyals2017starcraft}
O.~Vinyals, T.~Ewalds, S.~Bartunov, P.~Georgiev, A.~S. Vezhnevets, M.~Yeo,
  A.~Makhzani, H.~K{\"u}ttler, J.~Agapiou, J.~Schrittwieser, et~al.
\newblock Starcraft ii: A new challenge for reinforcement learning.
\newblock \emph{arXiv preprint arXiv:1708.04782}, 2017.

\bibitem[Wu et~al.(2017)Wu, Mansimov, Grosse, Liao, and Ba]{wu2017scalable}
Y.~Wu, E.~Mansimov, R.~B. Grosse, S.~Liao, and J.~Ba.
\newblock Scalable trust-region method for deep reinforcement learning using
  kronecker-factored approximation.
\newblock In \emph{Advances in neural information processing systems}, pages
  5279--5288, 2017.

\bibitem[Xu et~al.(2017)Xu, Modayil, van Hasselt, Barreto, Silver, and
  Schaul]{xu2017natural}
Z.~Xu, J.~Modayil, H.~P. van Hasselt, A.~Barreto, D.~Silver, and T.~Schaul.
\newblock Natural value approximators: Learning when to trust past estimates.
\newblock In \emph{Advances in Neural Information Processing Systems}, pages
  2120--2128, 2017.

\bibitem[Zhang et~al.(2018)Zhang, Ballas, and Pineau]{zhang2018dissection}
A.~Zhang, N.~Ballas, and J.~Pineau.
\newblock A dissection of overfitting and generalization in continuous
  reinforcement learning.
\newblock \emph{arXiv preprint arXiv:1806.07937}, 2018.

\end{thebibliography}

\clearpage

\appendix
\renewcommand\thefigure{\thesection\arabic{figure}}    
\setcounter{figure}{0}
\setcounter{section}{0}
\section{Appendix}
\subsection{Convergence Proof}
TD(0) is known to converge to the fixed point of the bellman operator \cite{sutton1988learning}:
\begin{equation}
\begin{split}
    \T \V(s_t) &= \expect_{s_{t+1} \sim \pol} [r(s_t) + \gamma \V(s_{t+1})]\\
\end{split}
\end{equation}
However, in practice we have access to a noisy version of the operator $\widetilde{\T}$ due to sampling process hence the noise term $w(t)$:
\begin{equation}
    w(t) = r_t + \gamma \V (s_{t+1}) - \expect_{s_{t+1} \sim \pol} [r + \gamma \V(s_{t+1})]
\end{equation}
\subsubsection{Derivation of $\widetilde{\V}$, Eq. \ref{decompose}}\label{decompose_app}
We take an example with $t=3$ and consider $i=2$:
\begin{equation}
    \begin{split}
        \vr(s_3) &= \param_3 \V(s_3) + (1-\param_3)\param_2 \V(s_2) + (1-\param_3)(1-\param_2)  \V(s_1)\\
        &= \V(s_2) - (1-(1-\param_3)\param_2)(\V(s_2) - \frac{\param_3 \V(s_3) + (1-\param_3)(1-\param_2)  \V(s_1)}{(1-(1-\param_3)\param_2)})\\
        &= \V(s_2) - (1-(1-\param_3)\param_2)(\V(s_2) - \widetilde{V}_t(s_i))
    \end{split}
\end{equation}
To see that $\widetilde{V}$ is a convex combination of the all the $\V$ encountered along the trajectory weight by $\param$ except $V(s_2)$ it suffices to see that:
\begin{equation}
    \begin{split}
        & \frac{\param_3  + (1-\param_3)(1-\param_2)  }{(1-(1-\param_3)\param_2)} = 1 \\
        &\equiv \param_3  + (1-\param_3)(1-\param_2)  = (1-(1-\param_3)\param_2) \\
        &\equiv \param_3 + (1-\param_3)\param_2 + (1-\param_3)(1-\param_2) = 1
    \end{split}
\end{equation}
where the last line is true because $\param \in (0,1]$
\subsubsection{Proof theorem 1}\label{proof}
\setcounter{theorem}{0}
\begin{theorem}
Let's define $V_{\max} = \frac{\widetilde{R}_{\max}}{1-(\gamma+(1-D))}$ and $V_{\min} = \frac{R_{\min}}{1-(\gamma-(1-D))}$. If the following holds
\begin{itemize}
    \item Let $X$ be the set of $\V$ functions such that $\forall s \in \mathbb{S} \quad  V_{\min} \leq \V(s) \leq V_{\max}$. We assume the functions are initialized in $X$.
    \item For a given $D$ and $\gamma$ we select C such that $(1-C)(V_{\max}-V_{\min}) \leq (1-D)V_{\min}$
\end{itemize}
then $\T^{\param}: X \rightarrow X$ is a contractive operator.
\end{theorem}
\begin{proof}
The first step is to prove that $\T^{\param}$ maps to itself for any noisy update $\widetilde{\T^{\param}}$. 
From 2) we know that
$(1-C) (V_{\max} - V_{\min}) < D V_{\min} \leq D V_{\max}$ we can then deduce that
\begin{equation}
    \begin{split}
        \widetilde{\T^{\param}}\V(s) &\leq \widetilde{R}_{\max} + \gamma V_{\max} +(1-C)(V_{\max} - V_{\min}) \\
        &\leq \widetilde{R}_{\max} + (\gamma+(1-D)) V_{\max} \\
        &\leq V_{\max}
    \end{split}
\end{equation}
and
\begin{equation}
    \begin{split}
        \widetilde{\T}^{\param}\V(s) &\geq R_{\min} + \gamma V_{\min} + (1-C)(V_{\min} - V_{\max})\\
        &\geq R_{\min} + (\gamma-(1-D)) V_{\min}\\
        &\geq V_{min}\\
    \end{split}
\end{equation}
The next step is to show that $\T^{\param}$ is a contractive operator:
\begin{equation}
\begin{split}
    &\norm{\T^{\param}V - \T^{\param}U}_{\infty}\\
    &\leq \max_{s,s'}\expect_{\pol}[\gamma V(s) + \Delta^V(s') - (\gamma U(s) + \Delta^U(s'))] \\
    &\leq \max_{s,s'}\expect_{\pol}[\gamma (V(s)-U(s)) + (1-D)(V(s') -U(s'))]\\
    &\leq \max_{s}\expect_{\pol}[((1-D)+\gamma)(V(s) - U(s))]\\
    &\leq ((1-D)+\gamma) \norm{V - U}_{\infty}
\end{split}
\end{equation}
and from the assumption we know that $(1-D)+\gamma < 1$.
\end{proof}
\subsubsection{Selecting C}\label{set_gamma}
To select C based on $\gamma$ and $D$ it suffice to solve analytically for:
\begin{equation}
\begin{split}
    &(1-C)(V_{\max} - V_{\min}) \leq (1-D)V_{\min} \\
    &\equiv (1-C)\frac{\widetilde{R}_{\max}}{1-(\gamma + (1-D)} \leq ((1-D)+(1-C)) \frac{R_{\min}}{1-(\gamma - (1-D)}\\
    &\equiv \frac{(1-C)(1-(\gamma-(1-D)))}{(1-(\gamma + (1-D))((1-D)+(1-C)} \widetilde{R}_{\max}\leq R_{\min} \\
    &\equiv \frac{D(1-C)(1-(\gamma-(1-D)))}{(1-(\gamma + (1-D))((1-D)+(1-C)} R_{\min} \leq R_{\min}
\end{split}
\end{equation}
which is satisfied only if:
\begin{equation}
\begin{split}
    \frac{D(1-C)(1-(\gamma-(1-D)))}{(1-(\gamma + (1-D))((1-D)+(1-C)} &\leq 1
\end{split}
\end{equation}
As an example for $D = 0.8$ and $\gamma = 0.5$ any $C\geq 0.33$ satisfies this inequality.
\subsubsection{Assumption asynchronous stochastic approximation}\label{assumption}
We now discuss the assumptions of theorem 3 in \cite{tsitsiklis1994asynchronous} 
\paragraph{Assumption 1:} Allows for delayed update that can happen in distributed system for example. In this algorithm all $\V$'s are updated at each time step $t$ and is not an issue here.
\paragraph{Assumption 2: } As described by \cite{tsitsiklis1994asynchronous} assumption 2 ``allows for the possibility of deciding whether to update a particular component $x_i$ at time $t$, based on the past history of the process.''. This assumption is defined to accommodate for $\epsilon$-greedy exploration in Q-learning. In this paper we only consider policy evaluation hence this assumptions holds.
\paragraph{ Assumption 3:} The learning rate of each state $s \in \mathbb{S}$ must satisfy Robbins Monroe conditions such that there exists $C \in \mathbb{R}$:
\begin{equation}
    \begin{split}
        \sum_{i=0}^{\infty} \alpha_t(s) 
e_t(s) &= \infty \quad \text{w.p.1}\\
        \sum_{i=0}^{\infty} (\alpha_t(s) e_t(s))^2 &\leq C      \end{split}
\end{equation}
This can be verified by assuming that each state gets visited infinitely often and an appropriate decaying learning rate based on $\#_s$ (state visitation count) is used (linear for example). 
\paragraph{Assumption 5:} This assumption requires $\T$ to be a contraction operator. This has been proven in theorem 1 of this paper.

\subsection
{Derivation of $\param$ update rule}
\label{beta_update_rule}

We wish to find $\param = \sigma(\omega)$ minimizing the loss : 

\begin{align}
    \min \ \ \frac{1}{2}(V^{\pol}(s_t) - \vrb(s_t))^2  \\
\end{align}
Taking the derivative of the R.H.S of 2 gives

\begin{align}
    \frac{d}{d\omega_{s_t}} \ \ &\bigg( \frac{1}{2}(V^{\pol}(s_t) - \vrb(s_t))^2  \bigg) =(V^{\pol}(s_t) - \vrb(s_t)) \big(\frac{d (V^{\pol}(s_t) - \vrb(s_t))}{d\omega_{s_t}} \big) 
     \big) )  \ \ \ \text{by chain rule}
\end{align}

We know that $\frac{d}{d\omega} \sigma(\omega_{s_t})=\sigma(\omega_{s_t})(1 - \sigma(\omega_{s_t}))$ \newline
and $\frac{d}{d\sigma(\omega_{s_t})} (V^{\pol}(s_t) - \vrb(s_t)) = \frac{d}{d\sigma(\omega_{s_t})} \bigg(\sigma(\omega_{s_t})\V(s_t) + \big(1 - \sigma(\omega_{s_t})\big)\vrb(s_{t-1}) - V^{\pol}(s_t)\bigg) = \V(s_t) - \vrb(s_{t-1})$ \newline \newline

Therefore, 
\begin{align}
    &\frac{d}{d\omega} \ \ \bigg( \frac{1}{2}(V^{\pol}(s_t) - \vrb(s_t))^2 + \lambda \sigma(\omega_{s_t}) \bigg) =\\
    &(V^{\pol}(s_t) - \vrb(s_t)) (\V(s_t) - \vrb(s_{t-1}))\bigg(\sigma(\omega_{s_t})(1 - \sigma(\omega_{s_t}))\bigg)
     + \lambda \bigg(\sigma(\omega_{s_t})(1 - \sigma(\omega_{s_t}))\bigg)  = \\
     &\bigg(\sigma(\omega_{s_t})(1 - \sigma(\omega_{s_t}))\bigg) \bigg((V^{\pol}(s_t) - \vrb(s_t))(\V(s_t) - \vrb(s_{t-1})) \bigg)
\end{align}

Finally, the update rule is simply a gradient step using the above derivative.
\subsection{Experiment}
\subsubsection{Policy Evaluation}
\label{Policy_evaluation_RVF}
\begin{figure}[h]
    \centering
    \includegraphics[width=0.5\textwidth, height=6cm]{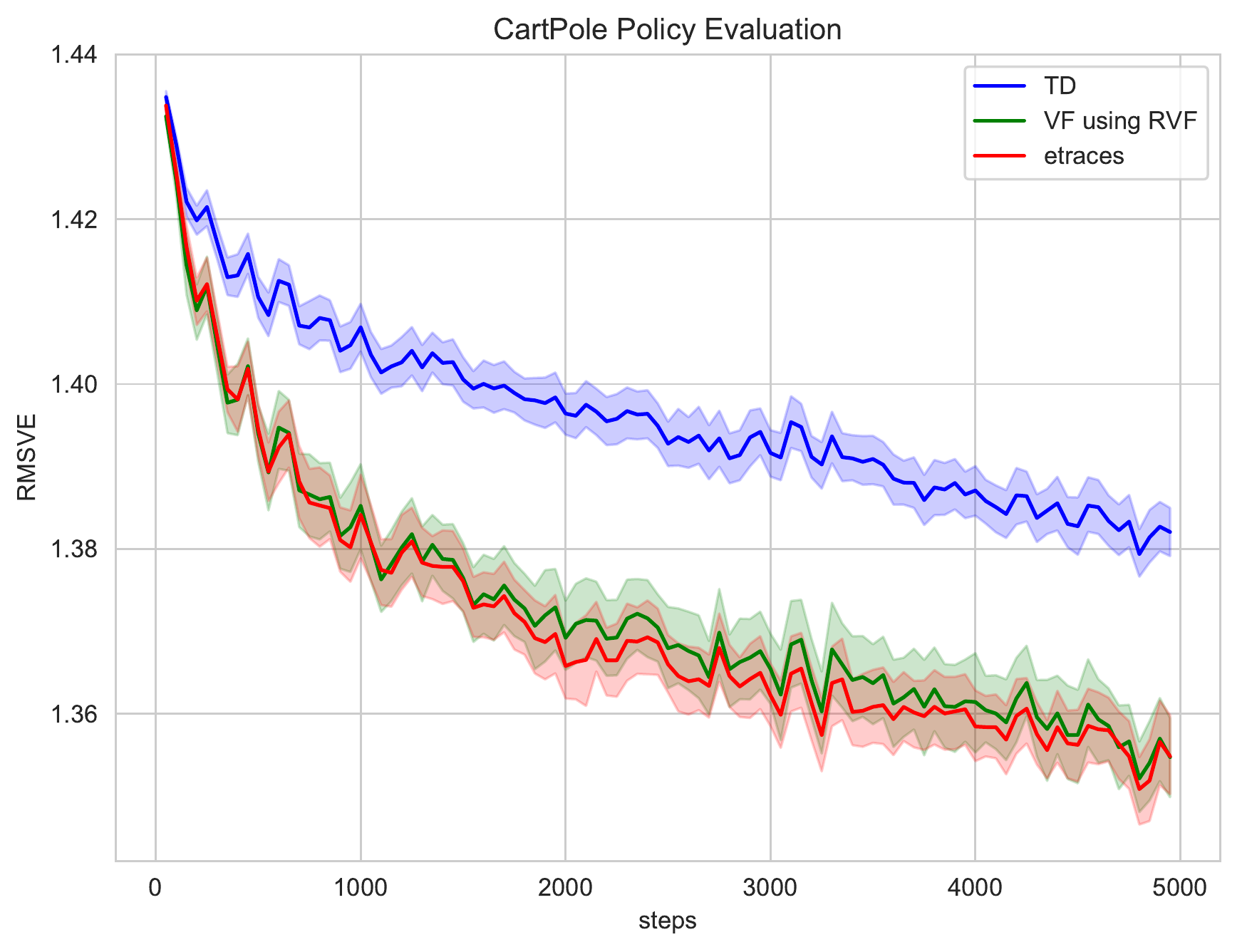}
    \caption{Comparison of RMSVE of various methods. Value Function estimated using RVF has lower error compared to Value Function estimated using TD and comparable to Eligibility traces}
    \label{fig:policy_eval}
\end{figure}
In this experiment, we perform a policy evaluation on CartPole task. In this environment the agent has to balance the pole on a cart. The agent can take one of the two actions available which would move the cart to either left or right. A reward of +1 is obtained for every time step and an episode is terminated if the cart move too far from the center or the pole dips below a certain angle.

In this task, we use a pretrained network to obtain the features that represent the underlying state. We train a linear function approximator using those features to estimate the value function. This experimental setup is similar to \cite{chung2018two} but the features in our case are obtained through a pretrained network instead of training a separate network. The samples are generated following a fixed policy and the same samples were used across all the methods to estimate value function. Each sample consists of 5000 transitions and the results are averaged over 40 such samples. We calculated the optimal value $V_\pi$ using a Monte Carlo estimate for 2000 random states following the same policy. We use the trained linear network to predict value function on these 2000 states. Once we get the predictions we calculate their Root Mean Square Value Error (RMSVE). The best hyperparameters were obtained through hyperparameter search for each method separately. The optimal learning rate was found to be $0.005$ for TD and RVF while $0.0005$ for eligibility traces when did a search in \{$0.0001, 0.0005, 0.001, 0.005, 0.01, 0.05$\}. The optimal beta learning rate was found to be $0.005$ when we searched in \{$0.001, 0.0001, 0.01$\} and the optimal lambda for eligibility traces was found to be $0.9$ when searched in \{$0.1, 0.4, 0.7, 0.9$\}. The RMSVE on various methods such as TD, eligibility traces - online TD($\lambda$) and the value functions of the state obtained using RVF algorithm are reported in Figure \ref{fig:policy_eval}. We notice that the the value function learned through RVF algorithm has approximately the same error as the value function learned through eligibility traces. Both RVF and eligibility traces outperform TD methods.
\subsubsection{Hyper-parameter TOY MDP}
\label{hyper_toy}
For every method the learning rate and $\lambda$ was tuned for optimal performance in the range $[0,1]$.\\
For RTD a learning rate of 0.5 for the value function and 1 for the beta function was found to be optimal with a lambda of 0.9.\\
For the GRU model we explored different amount of cell($\{1,5,10,15,20,25\}$) to vary the capacity of the model. The optimal number of hidden cell we found is 10, learning rate 0.5 and lambda 0.9. 
\subsubsection{Deep Reinforcement Learning}
\label{deep_RL_appendix}
The best hyperparameters were selected on 10 random seeds. 
\paragraph{PPO:} The following values were considered for the learning rate $\{3E-05,6E-05,9E-05,3E-04,6E-04\}$ and $N = \{2,5,10\}$.  The optimal values for learning rate is the same one obtained in the original PPO paper $3E-4$ and $N=5$.
We also compare with a larger network for PPO to adjust for the additional parameter of $\param$ the performance of vanilla PPO were found to be similar. In terms of computational cost, RVF introduce a computational overhead slowing down training by a factor of 2 on a CPU(Intel Skylake cores 2.4GHz, AVX512) compared to PPO. The results are reported on 20 new random seeds and a confidence interval of $68\%$ is displayed. 
\paragraph{A2C:} The following values were considered for the learning rate $\{1E-05,1E-04,1E-03,1E-02\}$ and the best learning rate was found to be $1E-04$. We tested bootstrapping on $5,10,15$ steps, we found no empirical improvements between $5$ and $10$ but noticed a significant drop in performance for bootstrapping $>10$. We believe that, it is due to the increase in variance of the target as we increase the steps of the bootstrapping. The best hyperparameters were found by averaging across 10 random seeds. The results reported are averaged across 20 random seeds. The confidence interval of $95\%$ is displayed.
\begin{figure}[h]
    \centering
    \includegraphics[width=\textwidth, height=4cm]{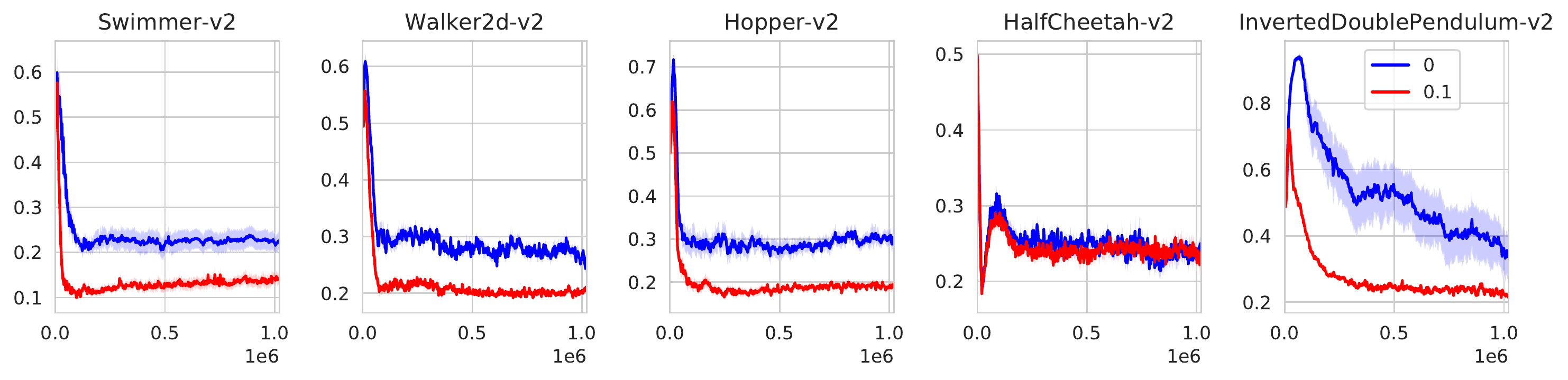}
    \caption{Mean beta values using recurrent PPO on Mujoco domains}
    \label{fig:beta_mean}
\end{figure}
\begin{figure}[h]
    \centering
    \includegraphics[width=\textwidth, height=4cm]{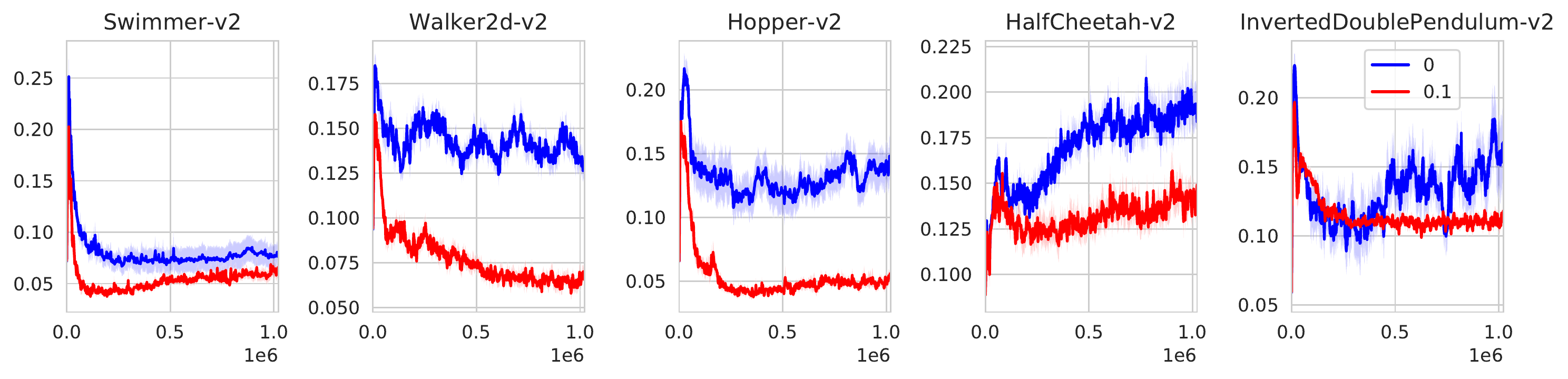}
    \caption{Standard deviation of beta using recurrent PPO on Mujoco domains}
    \label{fig:beta_std}
\end{figure}

\begin{figure}[h]
    \centering
    \includegraphics[width=\textwidth, height=4cm]{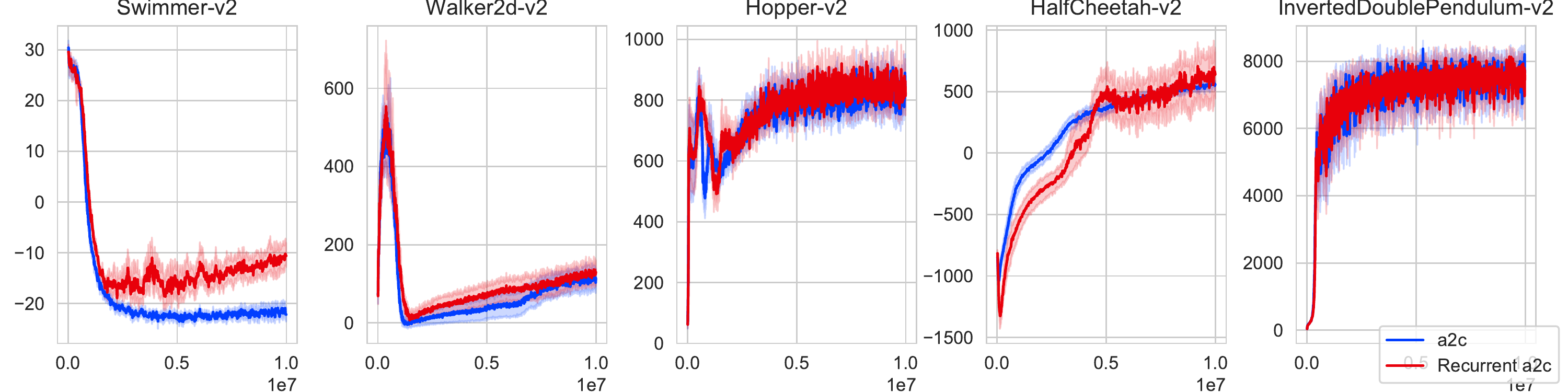}
    \caption{Performance of a2c and recurrent a2c on Mujoco tasks without noise in observations for 10M steps}
    \label{fig:a2c_no_noise_full}
\end{figure}
\begin{figure}[h]
    \centering
    \includegraphics[width=\textwidth, height=4cm]{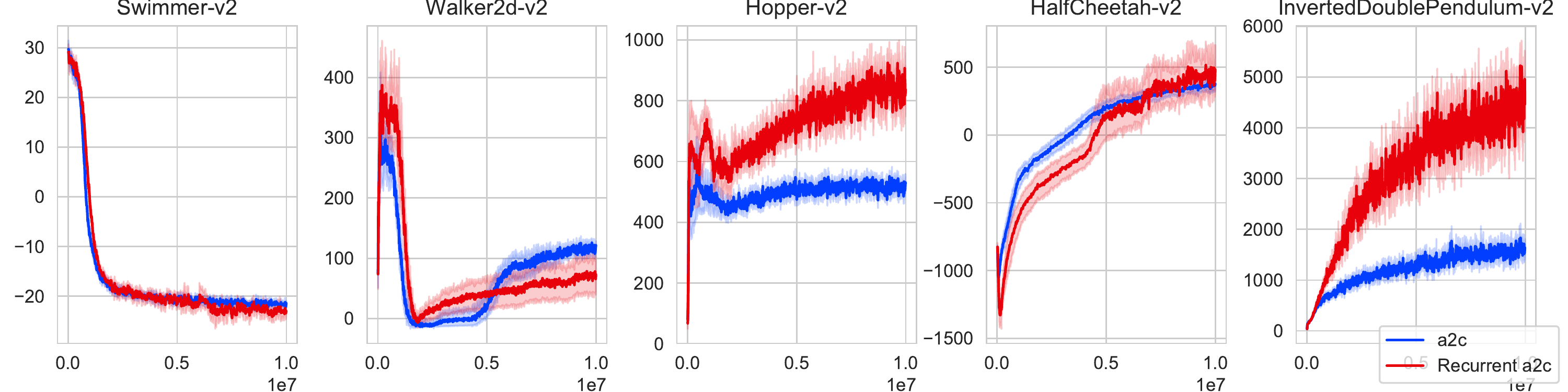}
    \caption{Performance of a2c and recurrent a2c on Mujoco tasks with a Gaussian noise(0.1) in observations for 10M steps}
    \label{fig:a2c_noise_full}
\end{figure}

\begin{figure}[H]
    \centering
    \includegraphics[width=\textwidth, height=4cm]{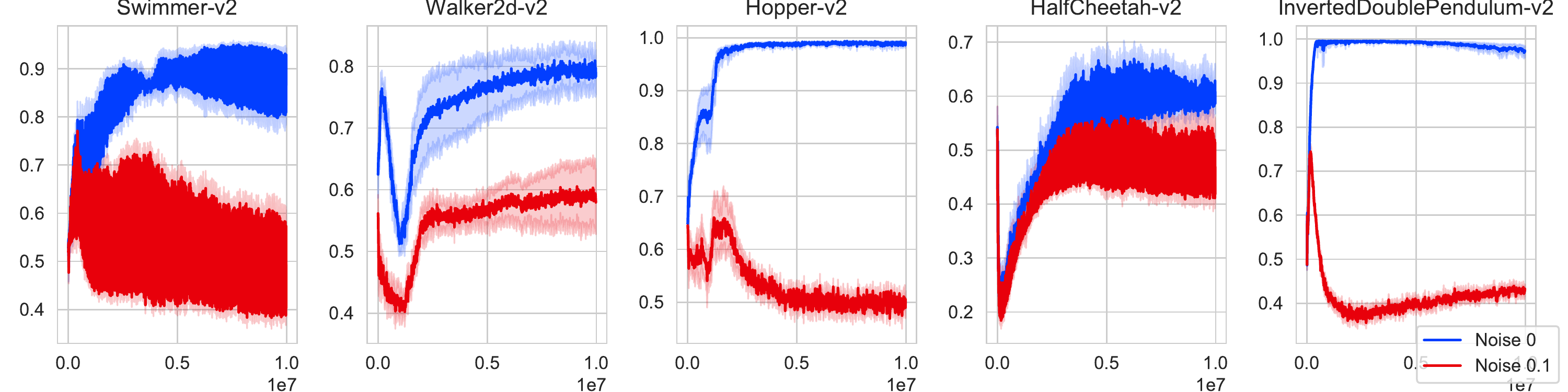}
    \caption{The mean of the emphasis function on various Mujoco tasks with and without noise plotted against the number of updates}
    \label{fig:beta_mean_a2c_mujoco}
\end{figure}

\begin{figure}[H]
    \centering
    \includegraphics[width=\textwidth, height=4cm]{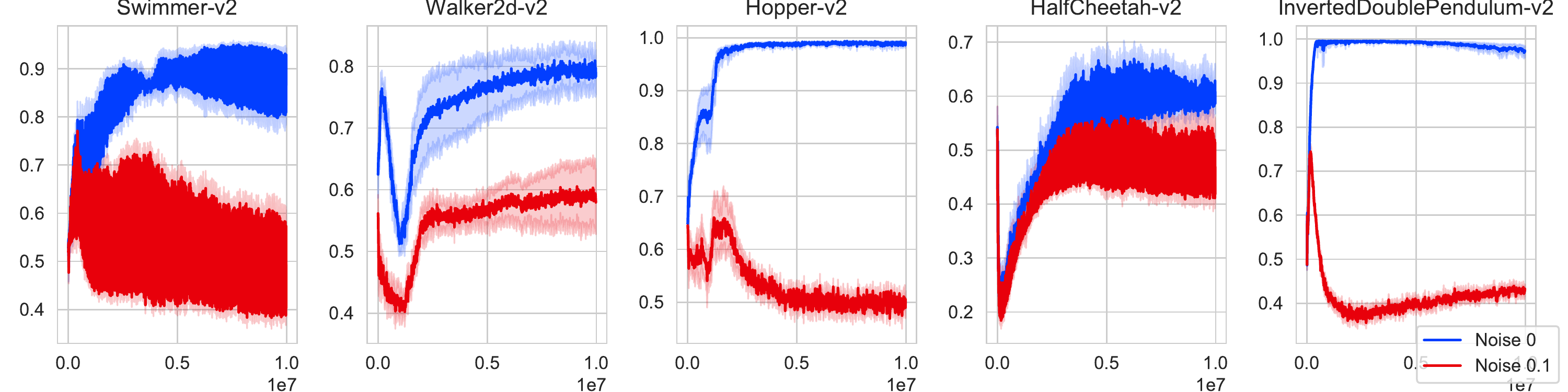}
    \caption{The standard deviation of the emphasis function on various Mujoco tasks with and without noise plotted against the number of updates}
    \label{fig:beta_std_a2c_mujoco}
\end{figure}
\end{document}